\documentclass{article}

\usepackage{graphicx}
\usepackage{amsmath}
\usepackage{amssymb}
\usepackage{booktabs}

\usepackage{microtype}
\usepackage{csquotes}
\usepackage{amsthm}
\usepackage{mathtools}
\usepackage{array}
\usepackage[acronym]{glossaries}
\usepackage[algoruled, english]{algorithm2e}
\usepackage{enumitem} 
\usepackage{url}
\usepackage{hyperref}

\makeatletter
\providecommand*{\diff}{\@ifnextchar^{\DIfF}{\DIfF^{}}}%
\def\DIfF^#1{%
	\mathop{\mathrm{\mathstrut d}}%
	\nolimits^{#1}\gobblespace}
\def\gobblespace{%
	\futurelet\diffarg\opspace}%
\def\opspace{%
	\let\DiffSpace\!%
	\ifx\diffarg(%
	\let\DiffSpace\relax%
	\else%
	\ifx\diffarg[
	\let\DiffSpace\relax%
	\else%
	\ifx\diffarg\{%
	\let\DiffSpace\relax%
	\fi\fi\fi\DiffSpace}%
\makeatother
\DeclarePairedDelimiterX{\abs}[1]{\lvert}{\rvert}{%
	\ifblank{#1}{\:\cdot\:}{#1}%
}
\DeclarePairedDelimiterX{\scprod}[2]{\langle}{\rangle}{%
	\ifblank{#1}{\:\cdot\:}{#1}%
	\,,\mathopen{}%
	\ifblank{#2}{\:\cdot\:}{#2}%
}
\DeclarePairedDelimiterX{\norm}[1]{\lVert}{\rVert}{%
	\ifblank{#1}{\:\cdot\:}{#1}
}
\DeclarePairedDelimiterX{\setc}[2]{\lbrace}{\rbrace}{%
	#1\ifblank{#2}{}{\,\delimsize\vert\,\mathopen{}#2}
}

\DeclarePairedDelimiterX{\set}[1]{\lbrace}{\rbrace}{%
	
	#1
}

\newcommand{\Real}{\mathbb{R}}

\newtheorem{proposition}{Proposition}

\newacronym{AAE}{AAE}{average angular error}
\newacronym{CTF}{CTF}{coarse-to-fine}
\newacronym{BP}{BP}{Blinn-Phong}
\newacronym{IID}{i.i.d.}{independent and identically distributed}
\newacronym{RLM}{RLM}{regularising Levenberg-Marquardt}
\newacronym{PS}{PS}{photometric stereo}
\newacronym{SfS}{SfS}{shape-from-shading}
\newacronym{IIE}{IIE}{image irradiance equation}

\usepackage[capitalize]{cleveref}
\crefname{section}{Sec.}{Secs.}
\Crefname{section}{Section}{Sections}
\Crefname{table}{Table}{Tables}
\crefname{table}{Tab.}{Tabs.}

\begin{document}

\title{On the Regularising Levenberg-Marquardt Method for
Blinn-Phong Photometric Stereo}

\date{}

\author{Georg Radow and Michael Breu\ss\\
	Chair of Applied Mathematics\\
	Brandenburg University of Technology Cottbus-Senftenberg\\
	{\tt\small \{radow,breuss\}@b-tu.de}
}
\maketitle

\begin{abstract}
Photometric stereo refers to the process to compute the 3D shape
of an object using information on illumination and reflectance 
from several input images from the same point of view.
The most often used reflectance model is the Lambertian reflectance,
however this does not include specular highlights in input images.
In this paper we consider the arising non-linear optimisation problem when 
employing Blinn-Phong reflectance for modeling specular effects.
To this end we focus on the regularising Levenberg-Marquardt scheme.
We show how to derive an explicit bound that gives information
on the convergence reliability of the method depending on given data,
and we show how to gain experimental evidence of numerical
correctness of the iteration by making use of the Scherzer condition.
The theoretical investigations that are at the 
heart of this paper are supplemented by some tests with real-world
imagery.
\end{abstract}

\section{Introduction}
\label{sec:introduction}

The \gls{PS} problem is a fundamental task in computer vision \cite{Horn1986}. 
The aim of \gls{PS} is to infer the 3D shape of an object from a 
set of multiple images. Thereby the images depict an object from 
the same perspective, but the illumination direction changes throughout the images. 
An important information besides the illumination is the light reflectance 
of the object. The classic \gls{PS} model \cite{Woodham1978,Woodham1980}
is formulated in terms of Lambertian light reflectance. 
A Lambertian surface is characterised by diffuse reflectance and the independence of
perceived shading from the viewing angle. The Lambertian set-up is certainly convenient 
for modeling, as it represents the most simple mathematical model 
for reflectance, and thus resulting formula and inverse problems are relatively 
simple. However, it is quite well known that in \gls{PS} specular highlights
\cite{Khanian2018} as well as non-Lambertian diffuse effects \cite{McGunnigle2012} 
may have an important impact on 3D reconstruction.
\par 
Let us also comment on some other basic characteristics of \gls{PS}. 
Depending on the knowledge on the lighting, one discerns 
between calibrated and uncalibrated \gls{PS}. 
In this work we consider only the calibrated case, where lighting directions 
and intensities are known. Furthermore, the final goal of \gls{PS} is to obtain a depth map, 
such that for each relevant image pixel three-dimensional information of the depicted object 
is obtained. While some approaches tackle this problem directly in terms of 
depth values \cite{Mecca-etal-2014}, the more common strategy is to divide depth computation
into two sub-problems. In doing so at first a map of normal vectors is computed, from which 
the (relative) depth is obtained in a second step. See for instance \cite{Queau2017b} 
for a survey on surface normal integration. 
In this paper we only consider the first of the latter tasks, 
that is to find the normal vectors.
Another aspect is sometimes the projection performed by the camera during image acquisition,
often leading to orthographic or perspective models, respectively. In this work
we address effectively both settings.

{\bf Our contribution.}
In this paper, we consider some theoretical aspects of practical value 
in the optimisation of \gls{PS} when using Blinn-Phong reflectance.
Here we extend in several ways upon previous work; let us especially 
refer to \cite{Khanian2018}, where the Blinn-Phong model is employed 
in a similar way as here. Thereby, we consider to include the potentially 
most important specularity parameter, the so-called shininess, as an unknown in the optimisation, 
which is in contrast to \cite{Khanian2018} and many other works in the field.
The approximate solution of the non-linear optimisation problem arising pixel-wise is performed
by the regularising Levenberg-Marquardt method, see especially \cite{Hanke1997}.
As this is an iterative method, it is important to assess the influence of initialisation 
on the convergence and to give a rigorous bound as a stopping criterion. Furthermore as the problem is non-linear, one can observe in practical examples, that it may be difficult to minimise the underlying residual. To address this issue we investigate the use of a \gls{CTF} scheme as well as an initialisation obtained through classical \gls{PS}. We show how to explore Scherzer's criterion~\cite{Hanke2010}, which appeared in \cite{Hanke1995} for the first time.
This criterion is considered for theoretical purposes within the construction of the 
method, in order to assess the convergence property in our \gls{PS} problem experimentally.

\section{Classical Photometric Stereo}
\label{sec:classicalPS}
Let us reiterate the classic \gls{PS} approach of 
Woodham~\cite{Woodham1978,Woodham1980}.
Given is a set of $m\geq 3$ 
images $\left(\mathcal{I}_{1},\dots,\mathcal{I}_{m}\right)^\top\eqqcolon\mathcal{I}$, so that $\mathcal{I}:\Omega\rightarrow\Real^{m}$,
along with the corresponding lighting directions  
$L_{k}\in\Real^{3}$ with 
$\norm*{L_{k}}=1$ for $k=1,\dots,m$, with associated intensities $l_{k}\geq0$. Throughout the paper $\norm{}$ denotes the Euclidean norm or the induced spectral norm. 
The object to be reconstructed is depicted usually as a non-rectangular domain 
$\Omega\in\Real^2$, which is embedded in the image domain.\par 
The surface normal vectors $\mathcal{N}:\Omega\rightarrow\Real^{3}$ with 
$\norm{\mathcal{N}(x,y)}=1$ for all $(x,y)^\top\in\Real^2$ and the albedo 
$\rho^d:\Omega\rightarrow\Real$ are fitted through a least squares approach, 
by minimising
\begin{equation}
	\iint_{\Omega}\norm*{ 
		\mathcal{R}^{\text{L}}(x,y)
		-\mathcal{I}(x,y)}^{2}\diff x\diff y,
	\label{eq:leastSquares}
\end{equation}
with reflectance function 
$\mathcal{R}^{\text{L}}\coloneqq\left(\mathcal{R}^{\text{L}}_{1},\dots, 
\mathcal{R}^{\text{L}}_{m}\right)^{\top}$, consisting of components
\begin{equation}
	\mathcal{R}^{\text{L}}_{k}\coloneqq
	\rho^{d}l_{k}L^{\top}_{k}\mathcal{N},\quad
	k=1,\dots,m.
	\label{eq:reflCont}
\end{equation}
In practice this boils down to finding a local solution $N\in\Real^3$ 
at every sample location $(x,y)^\top$ for the problem
\begin{equation}
	\min_{N}\norm*{LN-I}^2,\quad 
	L\coloneqq\begin{pmatrix}
		l_1L_1^\top \\
		\vdots\\
		l_mL_m^\top
	\end{pmatrix}, 
	\quad I\coloneqq\mathcal{I}(x,y).
\end{equation}
This, in turn, leads to the computation of the normal vectors and, as a byproduct, the albedo 
according to
\begin{equation}
	N=\left(L^\top L\right)^{-1}L^\top I,\qquad
	\rho^d(x,y)=\norm*{N},\qquad 
	\mathcal{N}(x,y)=N / \norm*{N} \,.
\end{equation}\par
%
\section{Blinn-Phong Photometric Stereo}
In the general least squares approach \cref{eq:leastSquares}, we can modify the reflectance function 
to account for non-Lambertian effects. To this end we investigate the \gls{BP} 
model~\cite{Phong1975,Blinn1977}, which has the form $\mathcal{R}^{\text{BP}}\coloneqq\left(\mathcal{R}^{\text{BP}}_{1},\dots, 
\mathcal{R}^{\text{BP}}_{m}\right)^{\top}$ with components
\begin{equation}
	\mathcal{R}^{\text{BP}}_{k}\coloneqq
	\rho^{d}l_{k}L^{\top}_{k}\mathcal{N}+ 
	\rho^{s}h_{k}\max\left\lbrace 0,\mathcal{H}^{\top}_{k}\mathcal{N}\right\rbrace^{\alpha},
	\label{eq:reflectanceBP}
\end{equation}
$k=1,\dots,m$. 
We observe by \eqref{eq:reflectanceBP} that in the \gls{BP} model,
diffuse reflection as in \eqref{eq:reflCont} is supplemented by a specular reflection term.
Here $\rho^{s}:\Omega\rightarrow\Real$ denotes the specular albedo. Another material 
parameter is the specular sharpness or shininess $\alpha:\Omega\rightarrow\Real$. The halfway 
vectors $\mathcal{H}_{k}:\Omega\rightarrow\Real^{3}$ depend on the viewing directions 
$\mathcal{V}:\Omega\rightarrow\Real^{3}$ and are computed for $k=1,\dots,m$ as
\begin{equation}
	\mathcal{H}_{k}(x,y)\coloneqq H_k/ \norm*{H_k}, \qquad 
	H_k \coloneqq L_k+\mathcal{V}(x,y).
\end{equation}
Making use of focal length $f$,
the viewing directions $\mathcal{V}^{\perp}$ and $\mathcal{V}^{\angle}$ 
in the orthographic and perspective setting respectively are
\begin{equation}
\mathcal{V}^{\perp}=(0,0,1)^{\top}, \qquad	\mathcal{V}^{\angle}(x,y)=(x,y,f)^{\top}.
\end{equation}
We reinterpret $l_{k}$ as diffuse intensity of the light source and 
denote $h_{k}\geq 0$ as specular intensity. 
To ensure that image intensities are only increased due to diffuse and specular terms, it 
is reasonable to enforce $\rho_{d},\rho_{s}\geq 0$. 
Furthermore $\rho_{d},\rho_{s}\leq 1$ ensures 
that at most as much image intensity is added as light intensity is supplied 
by each light source. 
Finally, it is reasonable to enforce $\alpha>1$ to actually 
produce specular highlights through the specular term.
\par
The \gls{BP} model was originally proposed for computer graphics. 
It is not based on physical laws, but it enables to create plausible images with a still simple model compared to other possible approaches.
Despite its simplicity, for use in 
inverse problems in computer vision, the non-linearities in \cref{eq:reflectanceBP} may 
pose considerable hurdles.\par 
Let us now discuss the modeling of the components in \cref{eq:reflectanceBP} along with a few adaptations we employ. First we turn our attention to the normal vectors $\mathcal{N}$. One may model them through derivatives of the depth or its logarithm. In this approach we may parametrise them at a specific location 
through depth derivatives $p,q$ as
\begin{equation}
\mathcal{N}(x,y) = \frac{N(p,q)}{\norm*{N(p,q)}}.
\label{eq:normalisation}
\end{equation}
However the step of obtaining a normal vector of length 1 in \cref{eq:normalisation} adds another layer of non-linearity to the model. In numerical experiments we found this approach to be not very reliable. Therefore we opt for an approach in analogy to classical \gls{PS}. In \cref{eq:reflectanceBP} we replace $\rho^d\mathcal{N}=N$ introducing the auxiliary variable $r=\rho^{s}/(\rho^{d})^\alpha$.
By furthermore replacing $\alpha=1+\exp(a)$ we ensure that $\mathcal{R}^\text{BP}$ has continuous first derivatives. \cref{eq:reflectanceBP} then takes the form
\begin{equation}
	\mathcal{R}^{\text{BP}}_{k}(N,r,a)=
	l_{k}L^{\top}_{k}N+ 
	rh_{k}\max\left\lbrace 0,\mathcal{H}^{\top}_{k}N\right\rbrace^{1+\exp(a)},
\end{equation}
with $r,a\in\Real$ and $N\in\Real^3$.
%
\section{On the Optimisation Strategy}
\label{sec:optimisation}
With \gls{BP} reflectance, we have to solve a non-linear least squares problem, 
to which end we utilise the \gls{RLM} scheme~\cite{Hanke1997,Hanke2010}. 
Writing the underlying task in standard notation, with this algorithm one may 
aim to find a solution $\vec{x}$ of the problem
\begin{equation}
F(\vec{x}) = \vec{y},\qquad F:\Real^n\rightarrow\Real^m,
\label{eq:rlmProblem}
\end{equation}
with a known differentiable function $F$. 
Let us note that the description and discussion of the \gls{RLM} algorithm 
in~\cite{Hanke2010} is in a more general setting. For simplicity we only 
give an overview of the algorithm based on finite dimensional spaces, 
as is fitting for the problem at hand.\par
It is furthermore assumed that the original data $\vec{y}$ is not known, 
but with some $\delta>0$ an estimate is required on how good the given 
data $\vec{y}^\delta$ approximates the original data, according to
\begin{equation}
\norm*{\vec{y}^\delta-\vec{y}}\leq\delta.
\label{eq:noiselevel}
\end{equation}
Then with some starting point $\vec{x}_0$ the iterative rule takes the form
\begin{equation}
    \vec{x}_{k+1}=\vec{x}_k+\left(F'(\vec{x}_k)^\top F'(\vec{x}_k)+\alpha_k I_n\right)^{-1}
    F'(\vec{x}_k)^\top\left(\vec{y}^\delta-F(\vec{x})\right)
\end{equation}
with Jacobian matrix $F'$, $n\times n$-dimensional identity matrix $I_n$ and a 
regularisation weight $\alpha_k>0$ such that with a preassigned 
$\rho\in(0,1)$ the new iterate $\vec{x}_{k+1}$ fulfils
\begin{equation}
    \norm*{\vec{y}^\delta-F(\vec{x}_k)-F'(\vec{x}_k)\left(\vec{x}_{k+1}-\vec{x}_k\right)} 
    =\rho\norm*{\vec{y}^\delta-F(\vec{x}_k)}.
    \label{eq:rho}
\end{equation}
The \emph{stopping criterion} of the \gls{RLM} scheme depends 
explicitly on the noise level $\delta$ in the given data. 
To stop at an iterate $\vec{x}_k$, it has to fulfil
\begin{equation}
\norm*{\vec{y}^\delta-F(\vec{x}_k)} \leq \tau\delta,
\label{eq:RLMstop}
\end{equation}
with a preassigned $\tau>2$, fulfilling $\rho\tau>1$. For numerical experiments we set $\rho=0.5,~\tau=2.5$, following~\cite{Hanke2010}.\par
The discussion of the \gls{RLM} scheme in \cite{Hanke2010} relies on the strong Scherzer condition \cite{Hanke1995}. For the Jacobian matrices at two points $\vec{x}_1,\vec{x}_2\in\Real^n$ there exists a matrix $R=R(\vec{x}_1,\vec{x}_2)$ such that
$F'(\vec{x}_1) = R F'(\vec{x}_2)$
and
\begin{equation}
    \norm*{R-I_m}\leq C^R\norm*{\vec{x}_1-\vec{x}_2}
    \label{eq:ScherzerConstant}
\end{equation}
with some $C^R>0$, which is constant for all $\vec{x}_1,\vec{x}_2\in\Real^n$. This condition imposes a certain regularity of the Jacobian matrix $F'$. In this context we are interested in a local approximation of $C^R$. For two consecutive iterations $\vec{x}_k,\vec{x}_{k+1}$ we estimate $R$ as a solution of 
$F'(\vec{x}_k) = R(\vec{x}_k,\vec{x}_{k+1}) F'(\vec{x}_{k+1})$
with minimal norm. Then we can locally approximate the constant in \cref{eq:ScherzerConstant} as
\begin{equation}
    C_k^{R,\text{loc}}=\frac{\norm*{R(\vec{x}_k,\vec{x}_{k+1})-I_m}}{\norm*{\vec{x}_k-\vec{x}_{k+1}}}.
    \label{eq:LocScherzerConstant}
\end{equation}\par 
Since $F$ in \cref{eq:rlmProblem} is nonlinear, we employ a \gls{CTF} framework. In doing so the data is scaled to a coarser scale, \emph{i.e.} to a lower resolution. The obtained result is then used as initialisation on the next finer scale, until we arrive at the original resolution.\par 
Let us focus on the assumption \cref{eq:noiselevel}. The noise level $\delta$ governs the 
stopping criterion of the \gls{RLM} scheme. If \cref{eq:noiselevel} is not fulfilled then the 
iterates may actually diverge.\par 
At this point we make the assumption that our data $\mathcal{I}(x,y)$ is a realisation of the \gls{BP} model corrupted by additive 
white Gaussian noise, \emph{i.e.}~it can be modelled as
\begin{equation}
	\mathcal{I}(x,y)=\mathcal{R}(x,y)+\varepsilon(x,y),\quad \text{for }(x,y)^{\top}\in\Omega.
\end{equation}
Here $\varepsilon(x,y)$ is a realisation of a multivariate normal distribution, such that the $m$ 
components are \gls{IID} with mean zero and standard deviation $\sigma>0$, the corresponding 
density function is
\begin{equation}
	f(X)=
	\frac{1}{\sqrt{2\pi}^{m}\sigma^{m}}\exp\left(-\frac{1}{2\sigma^{2}}\sum_{i=1}^{m}X_{i}^{2}\right),
	\label{eq:probDensity}
\end{equation}
\emph{cf.}~\cite{Prince2012}. The probability that \cref{eq:noiselevel} holds can be computed with the following result. The 
proof, which is technical but straightforward, 
is included for the readers convenience. The following result is also related to the Chi distribution.
\begin{proposition}
    \label{prop:1}
	Let $m\in\mathbb{N}$, $\delta>0$ and let $\varepsilon$ be a realisation of an $m$-dimensional 
	multivariate normal distribution with 
	mean zero, standard deviation $\sigma>0$ and density \cref{eq:probDensity}. The 
	probability of $P\coloneqq P(\norm{\varepsilon}\leq\delta|\sigma,m)$ can be computed as follows:
	\begin{enumerate}[label=(\roman{*}), ref=(\roman{*})]
		\item If $m$ is even, then
		\begin{equation}
			P=1-\exp\left(-\frac{\delta^{2}}{2\sigma^{2}}\right)
			\sum_{i=0}^{\frac{m}{2}-1}\left(\frac{\delta^{2}}{2\sigma^{2}}\right)^{i}\frac{1}{i!}.
			\label{eq:6prop6.1i}
		\end{equation}
		\item If $m$ is odd, then
		\begin{multline}
			P=\sqrt{\frac{2}{\pi}}
			\Bigg(
			\frac{1}{\sigma}\int_{0}^{\delta}\exp\left(-\frac{r^{2}}{2\sigma^{2}}\right)\diff r\\
			-\exp\left(-\frac{\delta^{2}}{2\sigma^{2}}\right)
			\sum_{i=1}^{\frac{m-1}{2}}
			\Bigg(\left(\frac{\delta}{\sigma}\right)^{m-2i}\prod_{j=1}^{\frac{m+1}{2}-i}\left(\frac{1}{2j-1}\right)\Bigg)\Bigg).
			\label{eq:6prop6.1ii}
		\end{multline}
	\end{enumerate}
\end{proposition}
\begin{proof}
	For any continuous probability density $f$ we have
	\begin{equation}
		P= P(\norm{\varepsilon}\leq\delta|\sigma,m)=
		\int_{\norm{X}\leq\delta}f(X)\diff X.
	\end{equation}
	Since the density function in \cref{eq:probDensity} is radially symmetric, this simplifies to
	\begin{equation}
		P=\int_{0}^{\delta}O_m(r)f(r,0,\dots,0)\diff r,
	\end{equation}
	where
	\begin{equation}
		O_m(r)= 2r^{m-1}\frac{\pi^{\frac{m}{2}}}{\Gamma\left(\frac{m}{2}\right)}
	\end{equation}
	denotes the surface area of a sphere with radius $r$ around the origin in $\Real^{m}$. $\Gamma$ 
	denotes the gamma function. Inserting \cref{eq:probDensity}, we write
	\begin{equation}
		P=\frac{2^{1-\frac{m}{2}}}{\sigma^{m}\Gamma\left(\frac{m}{2}\right)}
		\int_{0}^{\delta}r^{m-1}\exp\left(-\frac{r^{2}}{2\sigma^{2}}\right)\diff r.
		\label{eq:6prop6.1eq1}
	\end{equation}
	Since $\int r\exp(r^{2}/(2a))\diff r=a\exp(r^{2}/(2a))+c$, for $m>2$ the 
	integral in \cref{eq:6prop6.1eq1} can be simplified by partial integration, \emph{i.e.}
	\begin{multline}
		\int_{0}^{\delta}r^{m-2}\cdot r\exp\left(-\frac{r^{2}}{2\sigma^{2}}\right)\diff r\\
		=-\sigma^{2}\left[r^{m-2}\exp\left(-\frac{r^{2}}{2\sigma^{2}}\right)\right]_{r=0}^{\delta}
		+\sigma^{2}(m-2)\int_{0}^{\delta}r^{m-4}\cdot 
		r\exp\left(-\frac{r^{2}}{2\sigma^{2}}\right)\diff r.
	\end{multline}
	We now consider the two cases of $m$ being even or odd.\par 
	Let $m\in\mathbb{N}$ be even. Then repeated partial integration of the 
	integral \cref{eq:6prop6.1eq1} leads to 
	\begin{equation}
		\begin{split}
			&\int_{0}^{\delta}r^{m-1}\exp\left(-\frac{r^{2}}{2\sigma^{2}}\right)\diff r\\
			&=
			-\sum_{i=1}^{\frac{m}{2}-1}\sigma^{2i}\prod_{j=1}^{i-1}(m-2j)\left[r^{m-2i}\exp\left(-\frac{r^{2}}{2\sigma^{2}}\right)\right]_{r=0}^{\delta}\\
			&\quad 
			+\sigma^{m-2}\prod_{j=1}^{\frac{m}{2}-1}(m-2j)\int_{0}^{\delta}r\exp\left(-\frac{r^{2}}{2\sigma^{2}}\right)\diff
			r\\
			&=-\sum_{i=1}^{\frac{m}{2}}\sigma^{2i}\prod_{j=1}^{i-1}(m-2j)\left[r^{m-2i}\exp\left(-\frac{r^{2}}{2\sigma^{2}}\right)\right]_{r=0}^{\delta}\\
			&=-\sum_{i=1}^{\frac{m}{2}}\sigma^{2i}\frac{2^{i-1}\left(\frac{m}{2}-1\right)!}{\left(\frac{m}{2}-i\right)!}
			\left[r^{m-2i}\exp\left(-\frac{r^{2}}{2\sigma^{2}}\right)\right]_{r=0}^{\delta}\\
			&=\sigma^{m}2^{\frac{m}{2}-1}\left(\frac{m}{2}-1\right)!
			  -\sum_{i=1}^{\frac{m}{2}}\sigma^{2i}\frac{2^{i-1}\left(\frac{m}{2}-1\right)!}{\left(\frac{m}{2}-i\right)!}
			\delta^{m-2i}\exp\left(-\frac{\delta^{2}}{2\sigma^{2}}\right).
		\end{split}
		\label{eq:6prop6.1eq2}
	\end{equation}
	This formula can easily be verified for $m=2$, as in this case the initial integral simplifies to the form $\int r\exp(r^{2}/(2a))\diff r$. Inserting \cref{eq:6prop6.1eq2} and $\Gamma(m/2)=(m/2-1)!$ 
	into \cref{eq:6prop6.1eq1}, we obtain after an index shift \cref{eq:6prop6.1i}.\par 
	Now let $m\in\mathbb{N}$ be odd. Again we use repeated partial integration on the integral 
	in \cref{eq:6prop6.1eq1}, until we arrive at 
	\begin{multline}
			\int_{0}^{\delta}r^{m-1}\exp\left(-\frac{r^{2}}{2\sigma^{2}}\right)\diff r
			=\sigma^{m-1}\prod_{j=1}^{\frac{m-1}{2}}(2j-1)\int_{0}^{\delta}\exp\left(-\frac{r^{2}}{2\sigma^{2}}\right)\diff
			r\\-\sum_{i=1}^{\frac{m-1}{2}}\sigma^{2i}
			\frac{\prod_{j=1}^{\frac{m-1}{2}}(2j-1)}{\prod_{j=1}^{\frac{m+1}{2}-i}(2j-1)}
			\delta^{m-2i}\exp\left(-\frac{\delta^{2}}{2\sigma^{2}}\right).
	\end{multline}
	Plugging this together with 
	\begin{multline}
			\Gamma\left(\frac{m}{2}\right)=\Gamma\left(\frac{m-1}{2}+\frac{1}{2}\right)
		=\frac{\left(m-1\right)!\,\sqrt{\pi}}{\left(\frac{m-1}{2}\right)!\,2^{m-1}}\\
		=\frac{\prod_{j=1}^{\frac{m-1}{2}}\left((2j)(2j-1)\right)\sqrt{\pi}}
		{2^{\frac{m-1}{2}}\prod_{j=1}^{\frac{m-1}{2}}\left(2j\right)}
		=\frac{\sqrt{\pi}}{2^{\frac{m-1}{2}}}\prod_{j=1}^{\frac{m-1}{2}}(2j-1)
	\end{multline}
	into \cref{eq:6prop6.1eq1} leads to \cref{eq:6prop6.1ii}.
\end{proof}
%
\section{Experiments}
\label{sec:numExp}
%
\begin{figure}[tb]
\centering
\begin{tabular}{cccc}
	 \includegraphics[width=0.215\textwidth]{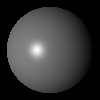}	
	 &\includegraphics[width=0.215\textwidth]{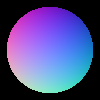}
	&\includegraphics[width=0.215\textwidth]{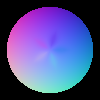}
	&\includegraphics[width=0.215\textwidth]{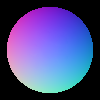}
\end{tabular}
	\caption{
	\emph{(left-to-right:)} one of the input images of the \emph{sphere}
	rendered using the \gls{BP} model;
	colour coded vector field of ground truth normal vectors; 
	classical \gls{PS} with preprocessing~\cite{Wu2010}, \gls{AAE} $1.02$; 
	developed \gls{BP} framework with \gls{CTF}, \gls{AAE} $0.37$
	\label{figure-new-1}
	}
\end{figure}

Since we focus on the computed 
vector fields of surface normals, it appears adequate to
employ colour coding of surface normals for visual assessment,
\emph{cf.}~Figure \ref{figure-new-1}. For quantitative evaluation we consider
here the standard \gls{AAE}, where the averaging 
is performed over the object domain. Let us note that we use the result obtained through 
classical \gls{PS} as an initialisation for the \gls{BP} model. Throughout the experiments we computed $\delta$ according to \cref{prop:1}, such that \cref{eq:noiselevel} is fulfilled with a probability of $95\%$. We observed that the choice of this confidence level is not critical for the outcome of our experiments.

\paragraph{Synthetic Test Example.}
As a synthetic experiment for our investigations 
we consider the \emph{sphere} example, see
Figure \ref{figure-new-1}. Let us note that we consider an 
orthographic setting for all the \emph{sphere} experiments.
As we observe in Figure \ref{figure-new-1}, in this experiment the developed
computational model and set-up enables to obtain a nearly
perfect result. For optimisation we employed in total $5$ input images,
of which we show here just one example.
For comparison, we give here the corresponding result obtained by Lambertian 
\gls{PS} applied at analogous input images where we filtered the specular 
highlights by the subspace technique proposed in \cite{Wu2010}, which is
supposed to make the input nearly Lambertian. As is confirmed here visually
as well as quantitatively, it appears favorable (at least in this example) 
to explore an explicit modeling like with the proposed \gls{BP} framework.

Let us note that in fact this test example may not be too easy, as can be observed
by the results obtained by preprocessing and Lambertian \gls{PS}. The reason 
is that the specular highlights in the input are not perfectly distributed 
over the sphere and may result in distortions if not being accounted for
sufficiently accurate in the model.

\paragraph{Evaluation of Scherzer's Condition.} As discussed in \cref{sec:optimisation}, between two iterates of the \gls{RLM} scheme 
we observe the local approximation $C_k^{R,\text{loc}}$ of the constant in \cref{eq:ScherzerConstant} 
according to \cref{eq:LocScherzerConstant}. As the Scherzer condition is an important assumption for the results in~\cite{Hanke2010}, we opt to add a break condition, 
where the algorithm stop if the estimate grows too large. In practice the algorithm is halted if we observe an iterate with $C_k^{R,\text{loc}}\geq 2000$. As can be seen in \cref{figure-Scherzer,figure-new-diligent} this is usually the case at locations where specular highlights may occur, as the angle between halfway vectors and surface normals becomes small. One may interprete this result 
in the way, 
that the energy that is minimised
features at highlights many 
small variations that makes it difficult to obtain 
a reliable local minimum.

We evaluated the restarting of the \gls{RLM} scheme with a larger parameter $\rho$ in \cref{eq:rho}, if it stopped before an iterate fulfils \cref{eq:RLMstop}. This may lead to a smaller trust region and to a more stable behaviour of the algorithm. However we did in general not observe 
a significant increase in quality. The results displayed here were thus computed without restarting the \gls{RLM} scheme, giving an account of the 
unstabilised version of the method.

\begin{figure}[tb]
\centering
\begin{tabular}{cc}
	 \includegraphics[width=0.215\textwidth]{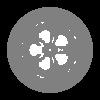}
	 \includegraphics[width=0.215\textwidth]{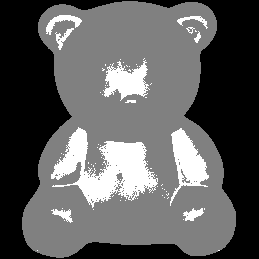}
\end{tabular}
	\caption{
	Algorithmic behaviour in the \emph{sphere} experiment \emph{(left)} and an example from the \emph{DiLiGent} data set~\cite{Shi2018} \emph{(right)}.
	White depicts the locations where the \gls{RLM} scheme stopped due to the $C_k^{R,\text{loc}}\geq 2000$ criterion.
	\label{figure-Scherzer}
	}
\end{figure}

\paragraph{Real World Test Example.}
In order to assess the properties and usefulness of the developed numerical
\gls{BP} framework, we exploit here a selected variety of 
examples taken from the \emph{DiLiGent} data set~\cite{Shi2018} which gives 
an account of photographed real-world objects with different 
reflectance properties. Here we do not employ a \gls{CTF} scheme, as we rely
on the initialisation obtained with classical \gls{PS}.
Let us note that the underlying model is now (in practice, weakly) 
perspective.

As can be visually assessed by means of Figure \ref{figure-new-diligent},
the proposed model along with its adaptations performs very reasonably
but in some details not perfect, depending on the actual example. 
For clarifying thereby the zones of influence of the specular terms
we depict masks showing the object parts where the \gls{BP} model
gives an effective contribution. When taking into account 
the properties of the considered examples, it appears especially
that the broad specularities as appearing in the input (teddy bear, goblet) 
may result in a certain inaccuracy. In turn, when highlights appear but 
are not too strong (cat, tea pot), results are quite convincing, given
that the underlying reflectance in these cases is supposed to be
non-linear in the diffuse reflectance as the underlying material 
is rough. In the tested real world setting from \emph{DiLiGent} 
the results are overall of similar quality to the preprocessed Lambertian 
method. Therefore we conjecture that our numerical \gls{BP} framework 
appears to be especially suited for dealing with objects with not too strong 
highlights, being at the same time able to tackle a certain range of
diffuse reflectance of rough materials.

\begin{figure*}[tb]
\centering
\begin{tabular}{cccc}
	\includegraphics[width=0.215\textwidth]{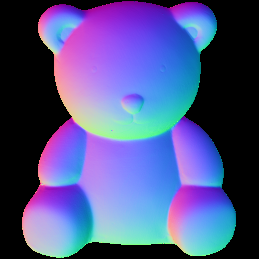}
	&\includegraphics[width=0.215\textwidth]{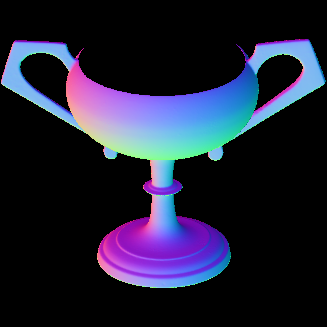}
	&\includegraphics[width=0.215\textwidth]{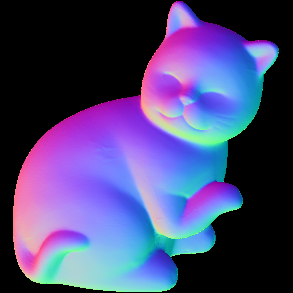}
	&\includegraphics[width=0.215\textwidth]{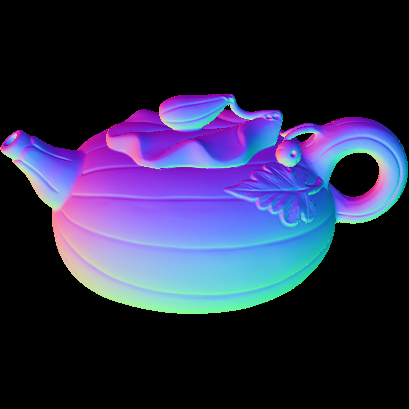}\\
	\includegraphics[width=0.215\textwidth]{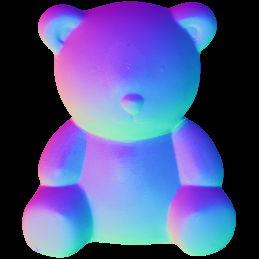}
	&\includegraphics[width=0.215\textwidth]{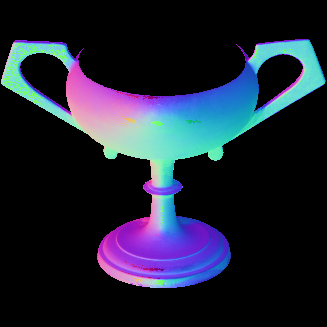}
	&\includegraphics[width=0.215\textwidth]{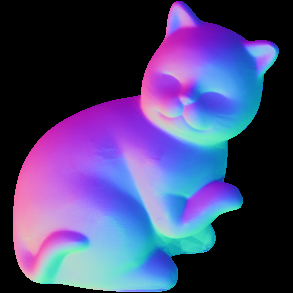}
	&\includegraphics[width=0.215\textwidth]{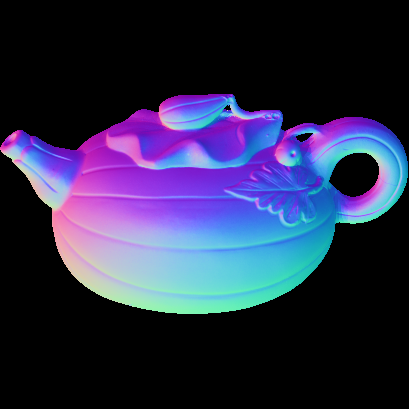}\\
	\includegraphics[width=0.215\textwidth]{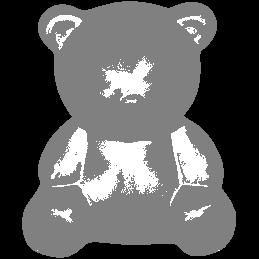}
	&\includegraphics[width=0.215\textwidth]{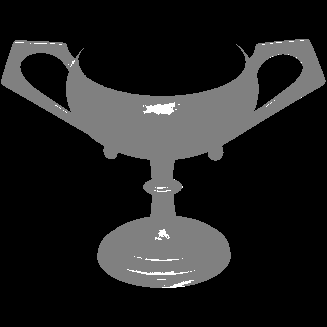}
	&\includegraphics[width=0.215\textwidth]{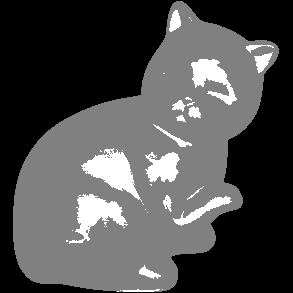}
	&\includegraphics[width=0.215\textwidth]{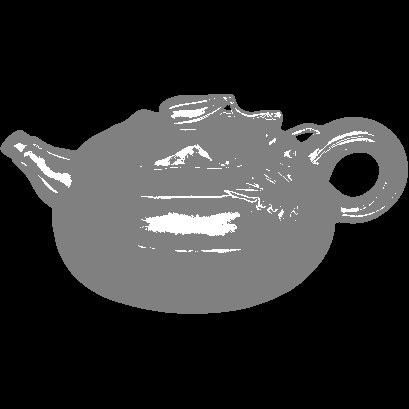}
\end{tabular}
	\caption{
	\emph{(left-to-right:)} Examples from \emph{Diligent} data sets. 
	\emph{(top-to-bottom:)} Visualisation of ground truth normals; 
	normal fields based on \gls{BP} (where we note the effect of the not satisfied Scherzer condition at some highlights at the goblet); 
	mask based on half directions. White depicts locations where the maximum of the cosines between halfway vectors and the normal vector obtained with classical \gls{PS} is $\geq0.99$.
	\label{figure-new-diligent}
	}
\end{figure*}
%

\section{Conclusion}
\label{sec:conclusion}
We discussed the \gls{BP} reflectance in the context of \gls{PS}. The augmentation of classical \gls{PS} with this reflectance model is straightforward, but solving the arising optimisation problem is less so. 
This task can be tackled with the \gls{RLM} scheme, which leads to satisfactory results.

The findings for the implementation of the \gls{RLM} scheme may be translated to other problems, since the assumption that the data follows a normal distribution is very common. The application of the \gls{BP} model to more complex data sets poses considerable hurdles, which may be adressed in future work.

{\small
\bibliographystyle{splncs04}
\bibliography{./refs}

\begin{thebibliography}{10}
\providecommand{\url}[1]{\texttt{#1}}
\providecommand{\urlprefix}{URL }
\providecommand{\doi}[1]{https://doi.org/#1}

\bibitem{Blinn1977}
Blinn, J.F.: Models of light reflection for computer synthesized pictures. In:
  Proceedings of the 4th annual conference on Computer graphics and interactive
  techniques - {SIGGRAPH} {\textquotesingle}77. {ACM} Press (1977)

\bibitem{Hanke1997}
Hanke, M.: A regularizing {L}evenberg - {M}arquardt scheme, with applications
  to inverse groundwater filtration problems. Inverse Problems  \textbf{13}(1),
   79--95 (1997)

\bibitem{Hanke2010}
Hanke, M.: The regularizing {L}evenberg-{M}arquardt scheme is of optimal order.
  Journal of Integral Equations and Applications  \textbf{22}(2),  259--283
  (2010)

\bibitem{Hanke1995}
Hanke, M., Neubauer, A., Scherzer, O.: A convergence analysis of the
  {L}andweber iteration for nonlinear ill-posed problems. Numerische Mathematik
   \textbf{72}(1),  21--37 (1995). \doi{10.1007/s002110050158}

\bibitem{Horn1986}
Horn, B.K.P.: Robot Vision. MIT Electrical Engineering and Computer Science,
  MIT Press (1986)

\bibitem{Khanian2018}
Khanian, M., Boroujerdi, A.S., Breu{\ss}, M.: Photometric stereo for strong
  specular highlights. Computational Visual Media  \textbf{4}(1),  83--102
  (2018)

\bibitem{McGunnigle2012}
McGunnigle, G., Dong, J., Wang, X.: Photometric stereo applied to diffuse
  surfaces that violate lambert's law. Journal of the Optical Society of
  America A  \textbf{29}(4), ~627 (mar 2012). \doi{10.1364/josaa.29.000627}

\bibitem{Mecca-etal-2014}
Mecca, R., Tankus, A., Wetzler, A., Bruckstein, A.M.: A direct differential
  approach to photometric stereo with perspective viewing. SIAM Journal on
  Imaging Sciences  \textbf{7}(2),  579--612 (2014). \doi{10.1137/120902458}

\bibitem{Phong1975}
Phong, B.T.: Illumination for computer generated pictures. Communications of
  the {ACM}  \textbf{18}(6),  311--317 (1975)

\bibitem{Prince2012}
Prince, S.J.: Computer Vision: Models, Learning, and Inference. Cambridge
  University Press (2012)

\bibitem{Queau2017b}
Qu{\'{e}}au, Y., Durou, J.D., Aujol, J.F.: Normal integration: A survey.
  Journal of Mathematical Imaging and Vision  \textbf{60}(4),  576--593 (2017)

\bibitem{Shi2018}
Shi, B., Mo, Z., Wu, Z., Duan, D., Yeung, S.K., Tan, P.: A benchmark dataset
  and evaluation for non-{L}ambertian and uncalibrated photometric stereo.
  {IEEE} Transactions on Pattern Analysis and Machine Intelligence pp. 1--14
  (2018)

\bibitem{Woodham1980}
Woodham, R.J.: Photometric method for determining surface orientation from
  multiple images. Optical Engineering  \textbf{19}(1),  134--144 (1980)

\bibitem{Woodham1978}
Woodham, R.J.: Photometric stereo: A reflectance map technique for determining
  surface orientation from image intensity. In: Nevatiam, R. (ed.) Image
  Understanding Systems and Industrial Applications. Proceedings of the Society
  of Photo-Optical Instrumentation Engineers, vol.~155, pp. 136--143. {SPIE}
  (1978)

\bibitem{Wu2010}
Wu, L., Ganesh, A., Shi, B., Matsushita, Y., Wang, Y., Ma, Y.: Robust
  photometric stereo via low-rank matrix completion and recovery. In: Asian
  Conference on Computer Vision (ACCV), Lecture Notes in Computer Science,
  vol.~6494, pp. 703--717. Springer Berlin Heidelberg (2010)

\end{thebibliography}
}

\end{document}